\documentclass{article}
\pdfoutput=1
\usepackage[
    textwidth=375pt,
    top=1in,
    headheight=12pt,
    headsep=25pt,
    footskip=30pt]{geometry}
\usepackage{hyperref}
\hypersetup{pdfauthor={Jonathan Cola\c{c}o Carr}, pdftitle={Conditions on Preference Relations that Guarantee the Existence of Optimal Policies}, hidelinks}
\usepackage{times}
\usepackage{amsmath}          
\usepackage{amssymb}          
\usepackage{amsfonts}         
\usepackage{amsthm}           
\usepackage{parskip}
\usepackage[T1]{fontenc}
\usepackage[dvipsnames]{xcolor}
\usepackage{import}
\usepackage{pdfpages}
\usepackage{transparent}
\usepackage{graphicx}
\usepackage{stackrel}

\theoremstyle{definition}

\newtheorem{assumption}{Assumption}
\newtheorem{definition}{Definition}
\newtheorem{theorem}[definition]{Theorem}
\newtheorem{lemma}[definition]{Lemma}
\newtheorem{proposition}[definition]{Proposition}
\newtheorem{corollary}[definition]{Corollary}
\newtheorem{example}[definition]{Example}

\renewenvironment{proof}{\hfill\break{\bfseries Proof. }}{\hfill\qed\break}

\renewcommand{\labelenumi}{\roman{enumi}.}

\usepackage[round]{natbib}
\makeatletter
\let\@fnsymbol\@arabic
\makeatother
\title{Conditions on Preference Relations that Guarantee the Existence of Optimal Policies}

\author{Jonathan Cola\c{c}o Carr\footnote{McGill University}\;$^,$\footnotemark[2]\and Prakash Panangaden\footnotemark[1]\;$^,$\footnote{Mila - Quebec AI Institute} \and Doina Precup\footnotemark[1]\;$^,$\footnotemark[2]}
\date{}
\begin{document}
\maketitle

\begin{abstract}
Learning from Preferential Feedback (LfPF) plays an essential role in training  Large Language Models, as well as certain types of interactive learning agents. However, a substantial gap exists between the
theory and application of LfPF algorithms.
Current results guaranteeing the
existence of optimal policies in LfPF problems assume that both the preferences and transition
dynamics are determined by a Markov Decision Process. We introduce the Direct Preference Process, a new framework for analyzing
LfPF problems in partially-observable, non-Markovian environments. Within
this framework, we establish conditions that guarantee the
existence of optimal policies by considering the ordinal structure of the
preferences. We show that a decision-making problem can have optimal policies -- that are characterized by recursive optimality equations -- even when no reward function can express the learning goal. These findings underline the need  to explore preference-based learning strategies which do not assume that preferences are generated by reward.
\end{abstract}

\section{Introduction}

Learning from Preferential Feedback (LfPF) is an important part of many real-world
applications of artificial intelligence (AI). At a high level, it describes an
interactive learning problem in which an agent's objectives are determined by
a collection of relative preferences over outcomes. 
LfPF has been used for a wide range of tasks, from robotics~\citep{Christiano2017_RLHF,Lee2021_Pebble} to the fine-tuning of Large Language Models~\citep{Bai2022,Openai2023gpt4,Rafailov2023direct,Stiennon_2020learning}.

However, the current theory of
LfPF lags far behind the success it has demonstrated in applications. There are no performance guarantees for LfPF problems beyond those defined through fully
observable, Markovian environments. Moreover, current
results~\citep{Chatterji_Neurips2021,Kong2022_provably,Saha2023_dueling,Xu2020_pbrl,Zhu2023_principled_rl} 
assume that the preferences in an LfPF problem are generated by an underlying reward
function. This assumption is known to be both unrealistic~\citep{Bobu2020_quantifying,Pandey2022,Tversky1974} and hard to
verify~\citep{Casper2023open}.  As a consequence, there are no 
theoretical guarantees for LfPF methods that are used in real-world scenarios.
\begin{figure}[htpb]
    \centering
    \includegraphics[width=0.5\textwidth]{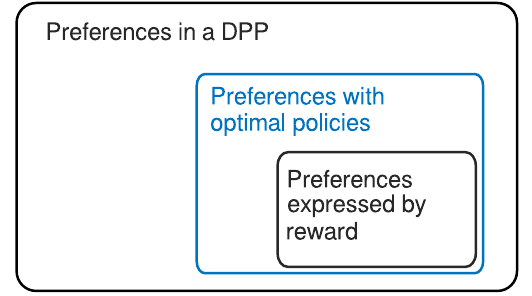}
    \caption{The Direct Preference Process (DPP) is a new framework for sequential decision-making from preferential feedback. The objectives in a DPP are given by a set of relative preferences over outcomes rather than a reward function. In Section~\ref{sec:optimal-policies} we show that it is possible for preferences in a DPP to have optimal policies even when no reward function can express the learning goal.}
    \label{fig:summary}
\end{figure}

In this paper, we define the Direct Preference Process, a 
model of preference-based learning in partially-observable, non-Markovian
environments. A key
feature of the Direct
Preference Process is that abstracts away the details of \textit{how} feedback is
given to a learning agent, instead working ``directly'' with the ordinal 
structure inferred from the preferences. This abstraction is particularly well suited for
LfPF problems, where a variety of feedback mechanisms are used during
training, including offline reward modelling~\citep{Ziegler2020finetuning,Bai2022}, once-per-episode trajectory
feedback~\citep{Chatterji_Neurips2021} and online feedback between trajectory segments~\citep{Christiano2017_RLHF}. 

\textbf{Our Contributions} 
\renewcommand{\labelenumi}{\roman{enumi}.}
\begin{itemize}
    \item We define the Direct Preference Process, a model of
        preference-based learning in partially-observable,
        non-Markovian environments (Section~\ref{sec:dpp}). We provide necessary and sufficient conditions
        that determine when a Direct Preference Process can be cast as an instance
        of reinforcement learning (RL).
    \item We show that optimal policies exist in a Direct Preference Process even when the preferences cannot be expressed by reward, and generalize the Bellman
    Optimality Equations to a larger class of order relations (Section~\ref{sec:optimal-policies}). In doing so, we highlight the properties of reward-based objectives that are \textit{not necessary} in order for optimal policies to exist.
    \item We derive conditions that determine when a
        computationally-constrained agent is able to behave optimally (Section~\ref{sec:bounded-optimal-policies}).
\end{itemize}
We focus on conditions that guarantee the existence of optimal
policies in order to determine when a LfPF problem has well-defined solutions. Our work opens up interesting areas of future research, both for the theory and practice of
preference-based learning.

\section{Related Work}
\label{sec:related-work}
In this section we review relevant sub-fields of LfPF.

\textbf{Preference-based RL.} Preference-based Reinforcement
Learning (PbRL)~\citep{Abdelkareem2022_PbRL_Review,Wirth2017_pbrl} describes a
collection of RL techniques used to solve sequential decision-making problems
whose objectives are determined by a set of relative preferences. This sub-field of LfPF includes RL from Human
Feedback~\citep{Christiano2017_RLHF,Ziegler2020finetuning} and RL from AI
Feedback~\citep{Bai2022}, both of which are popular methods of
fine-tuning Large Language Models. To the best of our
knowledge, the current results that guarantee the existence of optimal policies in
PbRL~\citep{Chatterji_Neurips2021,Kong2022_provably,Saha2023_dueling,Xu2020_pbrl,Zhu2023_principled_rl} rely on the assumption that there is an underlying controlled
Markov process and reward function which describe the transition dynamics and
preferences of the PbRL problem. We will relax both of these assumptions in
this paper, and instead analyze preference-based learning problems in terms of the ordinal structure
of the preferences.


\textbf{Ordinal Dynamic Programs.} Our work is reminiscent of ordinal dynamic
programs~\citep{Mitten1974,Sobel1975_odp,Weng2011Markov}. Our model is most similar to Mitten's Preference Order Dynamic
Program~\citep{Mitten1974}, which searched for conditions on the ordinal
structure of the objectives that could guarantee the existence of an optimal policy.
While Mitten assumed access to a set of preferences between ``intermediate policies'' for each state, we assume that the objectives
are given by a \textit{single} set of preferences between distributions over
trajectories, which seems like a more reasonable assumption given that feedback is
typically collected over trajectories or trajectory segments.

Our analysis significantly extends that
of Mitten and other prior ordinal dynamic programs. First, we abandon the
Markov assumption and consider problems that occur in partially observable,
non-Markovian environments. Second, we highlight two
structural properties (convexity and interpolation) as well as two concrete
examples (Examples~\ref{ex:consistent} and~\ref{ex:nointerpolation}) of goals that lead to the existence of
optimal policies in the absence of expected reward. These properties and examples
are essential to our theory, since they determine when it is impossible for ordinal
decision problems to be reconsidered as instances of reinforcement learning. To the best of our knowledge,
the connections between preference relations and expected reward in RL have only
recently started to be considered~\citep{Bowling2023settling,Shakerinava2022utility,Pitis2022rational}. Lastly, we provide additional
conditions which determine when it is possible for a computationally constrained
agent to behave optimally--an essential result for practical applications which has
not been studied in prior work.

\section{Background}
\label{sec:background}

Given a finite set $X$, let $\text{Dist}(X)$ be the set of probability distributions over
$X$. For distributions $A$ and $B$ over $X$ and non-negative number $\alpha$ less than or equal to one, the distribution $\alpha A + (1-\alpha)B$
assigns the probability of an element $x\in X$ as $\alpha A(x) + (1-\alpha)B(x)$.

We interpret a binary relation $\preceq$ on
$\text{Dist}(X)$ as a set of relative preferences, so that for any two
distributions $A$ and $B$ over $X$, the statement ``$A \preceq B$''
means that $B$ is at least as desirable as $A$. Outcome $B$ is ``strictly
preferred'' to $A$, written $A \prec B$, if $A\preceq B$ and $\lnot(B \preceq
A)$. Outcomes $A$ and $B$ are ``$\preceq$-equivalent'', written $A\sim B$, if both $A \preceq B$
and $B \preceq A$. 
\begin{definition} A binary relation $\preceq$ on $\text{Dist}(X)$ is a \textbf{preorder} if it satisfies both of the
    following properties: 
    \begin{itemize}
        \item \textit{(reflexivity)} for any distribution $A$ over $X$, $A
            \preceq A$. 
        \item \textit{(transitivity)} for any three distributions $A$,$B$ and
            $C$ over $X$, if $A \preceq B$ and $B \preceq C$ then $A \preceq
            C$.
    \end{itemize}
    A preorder is \textit{total} if for any two distributions $A$, $B$ over $X$,
    either  $A \preceq B$ or $B \preceq A$.
\end{definition}

\subsection{Agents and Environments}
To describe interactions between a learning agent and its environment, we consider a finite
version the agent-environment interface~\citep{Abel2023_definition}.
This framework draws from related models of 
partially-observable, non-Markovian learning
problems~\citep{Dong2021_simple_agent,Lu2023_BitbyBit,Hutter2016_esa}. 
\begin{definition}
    An \textbf{agent-environment interface} $(\mathcal{O},\mathcal{A},T)$ consists of
    a finite set of observations $\mathcal{O}$, a finite set of actions
    $\mathcal{A}$ and a time horizon $T\in \mathbb{N}$.
\end{definition}
The finite horizon assumption is motivated by the fact that in practice, human
labelers rank trajectories, so only a finite number of time steps is available 
to train. However, we impose no restrictions on the transition dynamics, so 
there may be an arbitrary (but finite) number of training episodes. To avoid trivialities we  assume that both the action and observation
sets are non-empty. Extensions to infinite action and observation sets is left as an important area for future work. 

For an agent-environment interface $(\mathcal{O},\mathcal{A},T)$, a set
of \textit{$t$-histories} is defined for each non-negative integer $t$ less than or equal to $T$ as follows: $\mathcal{H}_0= \mathcal{O}$ and $\mathcal{H}_{t+1}= \mathcal{H}_t \times
(\mathcal{A}\times \mathcal{O})$. We define $\mathcal{H}$ as the set of all histories,
\begin{equation}
\mathcal{H}= \bigcup_{t=0}^{T}\mathcal{H}_t 
.\end{equation} 
We will refer to histories of length $T$  as
\textit{trajectories} and write $\Omega$ instead of $\mathcal{H}_T$.
For each non-negative integer $t$ less than or equal to $T$, the projection
$\xi_{0:t}:\Omega\to \mathcal{H}_t$ maps each trajectory to its first sub-history of
length $t$. The
environment determines which histories are attainable in a given learning problem. 
\begin{definition}
    An \textbf{environment} with respect to the interface $(\mathcal{O},\mathcal{A},T)$ is a
    tuple $e=(\rho_0,\rho)$ consisting of an initial distribution over observations
    $\rho_0\in \text{Dist}(\mathcal{O})$ and a
    transition probability function $\rho: (\bigcup_{t=0}^{T-1}\mathcal{H}_t)\times
    \mathcal{A}\to \text{Dist}(\mathcal{O})$. 
\end{definition}
Notice that the transition dynamics in a learning environment may depend on the entire
history, which may be the case in practice.

\begin{example}[Generative Language Models 1]\label{ex:llm-1}
An agent-environment
interface $(\mathcal{O}, \mathcal{A},T)$ can describe the interactions between
a language model and a user, where the set of observations consists of the possible
messages the user sends to the language model and the set of actions consists of the possible
messages the model is able to send to the user. The environment $e$
models the user's question patterns and prompts, which may depend on the full
conversation history.
\end{example}
The behaviour of an agent is defined by its policy.
\begin{definition}
    A \textbf{policy} $\pi$ with respect to interface $(\mathcal{O},\mathcal{A},T)$ is a function
    $\pi: \mathcal{H}\to \text{Dist}(\mathcal{A})$.
\end{definition}
It is important to allow policies to depend on the full history in Section~\ref{sec:optimal-policies} because we seek optimality conditions that do not depend on what information is available to the agent.
We address agents with memory constraints in Section~\ref{sec:bounded-optimal-policies}, where the decision problem becomes partially-observable and non-Markovian. This is typical when function approximation is used.

\textbf{Important Distributions.} Agents will be 
evaluated according to the distributions
their policies induce over $\Omega$. For each policy $\pi$ and history
    $h_t$, we define $D^{\pi}(h_t)$ as the distribution over $\Omega$ induced by
    starting from history $h_t$ and following $\pi$ in environment $e$ thereafter.
    More precisely, for each trajectory $h_T$, $D^{\pi}(h_T)$ is equal to the Dirac
    distribution concentrated at $h_T$ and for each history $h_t$ of length less
    than $T$,
\begin{equation}
    D^{\pi}(h_t) = 
    \sum_{a\in \mathcal{A}}\pi(a \vert h_t)\sum_{o\in \mathcal{O}}\rho(o \vert
    h_t,a) D^{\pi}(h_t\cdot (a,o)),
\end{equation}
where $h_t\cdot (a,o)$ is the history of length $t+1$ obtained by appending the action-observation pair
$(a,o)$ to $h_t$. Similarly, we define $D^{\pi}(h_t\cdot a)$ as the distribution over $\Omega$ induced
by starting from
history $h_t$, selecting action $a$ and following $\pi$ thereafter. Note that we are
overloading notation here; 
$D^{\pi}$ may take either a history or a history appended with an action as 
its argument. 
The distributions $D^{\pi}(h_t)$, $D^{\pi}(h_t\cdot a)$ and $D^{\pi}(h_t\cdot
(a,o))$ are related as follows:
\begin{subequations}
    \label{eq:dpi}
\begin{align}
    D^{\pi}(h_t\cdot a)&= \sum_{o\in \mathcal{O}} \rho(o \vert h_t,a)
    D^{\pi}(h_t\cdot (a,o))\\
    D^{\pi}(h_t) &= \sum_{a\in \mathcal{A}} \pi(a \vert h_t) D^{\pi}(h_t\cdot a)
.\end{align}
\end{subequations}
\textbf{Attainable Histories.} As
in \citet{Abel2023_definition}, we will only consider the performance of
policies in histories that occur can with non-zero probability in a given environment
$e$ under some policy. For each non-negative integer $t$ less than or equal to $T$, the \textit{set of attainable
$t$-histories in $e$}, denoted by $\mathcal{H}_t^{e}$, is defined recursively as
follows:
$\mathcal{H}_0^{e}$ is equal to the support of $\rho_0$ and 
\begin{equation}
\label{eq:realiz}
\mathcal{H}_{t+1}^{e} = \{h_{t}\cdot (a,o)\in \mathcal{H}_{t+1}:\; h_{t}\in
\mathcal{H}_{t}^{e} \text{ and }\rho(o \vert h_t,a)>0\}
.\end{equation}
We define $\mathcal{H}^{e}= \bigcup_{t=0}^T \mathcal{H}_t^{e}$ as
the \textit{set of attainable histories in $e$} and $\Omega^{e}=
\mathcal{H}_T^e$ as the \textit{set of attainable trajectories in $e$}.

\section{The Direct Preference Process}
\label{sec:dpp}
An agent-environment interface, an environment and a binary relation on
the set of distributions over trajectories define a Direct Preference
Process. 
\begin{definition}\label{defn:dpp}
    A \textbf{Direct Preference Process} $(\mathcal{O}, \mathcal{A}, T,e, \preceq)$ 
    consists of an agent-environment interface $(\mathcal{O}, \mathcal{A}, T)$, an
    environment $e$ and a binary
    relation $\preceq$ on the set of distributions over $\Omega$.
\end{definition}
The distinctive feature of the Direct Preference Process
is that the preference relation $\preceq$ defines the goals of
a learning problem. Importantly, we do not assume that these objectives have any quantitative structure. However, when a numerical objective function does
convey the goals of a decision problem, there is an implicit Direct Preference Process. 
\begin{example}[Generative Language Models 2]\label{ex:llm-2}
  Given the interface $(\mathcal{O},\mathcal{A}, T)$ and environment $e$ from Example~\ref{ex:llm-1}, the goal of the language model may be to maximize a performance
    function $\varphi: \text{Dist}(\Omega)\to \mathbb{R}$. This induces a preference
    relation $\preceq_{\varphi}$ on $\text{Dist}(\Omega)$, defined for each pair of distributions $A$ and $B$ over $\Omega$ as:
    \[
        A \preceq_{\varphi} B \iff \varphi(A)\le \varphi(B)
    .\] 
    The Direct Preference Process $(\mathcal{O}, \mathcal{A}, T, e,\preceq_{\varphi})$ underlies this decision problem.
\end{example}

A policy $\pi$ is optimal in a Direct Preference Process if it achieves
the most desirable outcome in every attainable start history.
\begin{definition}
    \label{defn:opt}
    Given a Direct Preference Process $(\mathcal{O}, \mathcal{A}, T,e, \preceq)$, a
    policy $\pi$ is \textbf{$\preceq$-optimal} (or simply \textbf{optimal}) if for every attainable history
    $h_t$ and policy $\pi'$, $D^{\pi'}(h_t) \preceq D^{\pi}(h_t)$. 
\end{definition}
In Example~\ref{ex:llm-2}, the language model's policy is optimal for a given user if it achieves the best performance in every attainable conversation history.

\subsection{Reward-Based Objectives}
\label{sec:dpp-rl}
As noted in Section~\ref{sec:related-work}, the current analyses of PbRL problems
assume that preferences are derived from an underlying reward function. While
ordinal dynamic programs do not make this assumption outright, it is unclear 
whether or not an underlying reward is implied by the assumptions made about the preferences. In
contrast to both of these models, the Direct Preference Process comes with necessary and sufficient conditions that determine when goals can
be expressed by the expected cumulative sum of numerical rewards.
\begin{definition}\label{defn:erc}
Let $(\mathcal{O},\mathcal{A},T, e, \preceq)$ be a Direct Preference Process. We say
that $\preceq$ is expressed by the \textbf{expected reward criterion} if there is a
function $r: \mathcal{H}\to \mathbb{R}$ such that for any two distributions $A$ and
$B$ over $\Omega$,
\begin{equation}
    \label{eq:erc}
  A \preceq B \iff
    \mathbb{E}_{A}\left[ \sum_{t=0}^{T} r(H_t) \right]\le \mathbb{E}_{B}\left[
    \sum_{t=0}^{T}r(H_t) \right]
.\end{equation} 
We say that $r$ \textbf{expresses} $\preceq$ if~(\ref{eq:erc})
holds for any two distributions $A$ and $B$ over $\Omega$. 
\end{definition}
As an important sanity check, Theorem~\ref{thm:rl-opt} confirms that when goals are expressed by the expected reward criterion, the previous definition
of an optimal policy  can be re-stated in terms
the value function criterion found in the RL
literature~\citep{Sutton2018_rl_book,Puterman_MDPs}. For a reward function $r: \mathcal{H}\to \mathbb{R}$, we define the \textit{$r$-value} of a
policy $\pi$ in history $h_t$ as 
\begin{equation}
    \label{eq:value}
    V_{\pi}(h_t;r)= \mathbb{E}_{\pi}\left[ \sum_{s=t}^{T} r(H_s) \vert H_t=h_t \right]
,\end{equation} 
where the conditional expectation is taken with respect to $D^{\pi}(h_t)$.
\begin{theorem}
    \label{thm:rl-opt}
    Let $(\mathcal{O},\mathcal{A},T,e, \preceq)$ be a Direct Preference Process and
    suppose that a reward function $r:\mathcal{H}\to \mathbb{R}$ expresses
    $\preceq$. A policy $\pi$ is $\preceq$-optimal if and only if for each attainable history $h_t$, \[
        V_{\pi}(h_t; r)= \sup_{\pi'}V_{\pi'}(h_t; r)
    .\] 
\end{theorem}
\begin{proof}
    Immediate from Definitions~\ref{defn:opt} and~\ref{defn:erc}.
\end{proof}

As a consequence of Theorem~\ref{thm:rl-opt}, the standard RL problem can be seen as
a Direct Preference Process, where the performance of a policy is
only considered on histories that are attainable in an environment. A natural next question is: \textit{what kinds of
Direct Preference Processes can be cast as RL problems?} Stated in
Theorem~\ref{thm:vnm}, the von Neumann-Morgenstern (vNM) Expected Utility Theorem~\citep{vonneumann1947} provides a
decisive answer to this question. Their result depends on the
following properties.

\begin{definition} Let $X$ be a finite set. A total preorder $\preceq$ on the set
    of distributions over $X$ is said to satisfy:
    \begin{enumerate}
        \item \textbf{consistency} (or is \textbf{consistent}) if for every $\alpha\in (0,1)$ and any
            distributions $A$, $B$ and $C$ over $X$, $A \preceq B$ implies \[
                \alpha A+(1-\alpha)C \preceq \alpha B+(1-\alpha)C
            .\] 
            
        \item \textbf{convexity} (or is \textbf{convex}) if for every $\alpha\in (0,1)$ and any distributions $A$, $B$ and $C$ over
            $X$, $A \preceq B$ if and only if \[
                \alpha A+ (1-\alpha)C \preceq \alpha B+(1-\alpha)C
            .\] 
    \item \textbf{interpolation} if for any distributions $A$,$B$ and $C$ over
        $X$, if $A
        \preceq B$ and $B\preceq C$ then there exists $\alpha\in [0,1]$  such that \[
             \alpha A+(1-\alpha)C \sim B
        .\] 
    \end{enumerate}
\end{definition}
The following example clarifies the difference between consistency and
convexity, drawing from a scenario with unacceptable risk~\citep{Jensen2012}.
\begin{example}\label{ex:consistent}
Let $E$ be a proper non-empty subset of $\Omega$, interpreted as an event of
``unacceptable risk''. Given a real-valued function
$u$ on $\Omega$ and a real number $\beta$ such that such that $u$ is strictly greater
than $\beta$ on $\Omega$, define the performance $\varphi: \text{Dist}(\Omega)\to
\mathbb{R}$ as: \[
    \varphi(A)= \begin{cases}
        \sum_{\omega\in \Omega}u(\omega)A(\omega)& A(E)=0\\
        \beta e^{A(E)}&A(E)>0,
    \end{cases}
\] 
where $A(E)$ is the probability of event $E$ under $A$. Assuming that  $u$ is
non-constant on the complement of $E$, the relation $\preceq_{\varphi}$ on $\text{Dist}(\Omega)$ defined by \[
    A \preceq_{\varphi} B \iff \varphi(A)\le \varphi(B)
,\] 
is a total consistent preorder that is not convex. Moreover, $\preceq_{\varphi}$ does not
satisfy interpolation.

\begin{proof}
   The relation $\preceq_{\varphi}$ is a total preorder because it inherits the
   totality and transitivity of `$\le$' on $\mathbb{R}$. Let $\alpha$ be a
   positive number less than one and $A,B$ and $C$ be distributions over $\Omega$.
   If $A \preceq_{\varphi} B$ then the probability of $E$ under $A$ is greater than
   or equal to the probability of $E$ under $B$. Therefore, if either $A(E)$ or $C(E)$ is
   positive, then \[
       0< \alpha A(E)+ (1-\alpha)C(E) \le \alpha B(E)+(1-\alpha)C(E)
   ,\] 
   from which it follows that $\alpha A+(1-\alpha)C \preceq_{\varphi} \alpha B+(1-\alpha)C$.
   Otherwise, if both $A(E)$ and $C(E)$ are zero then $B(E)$ is also zero, and so\begin{align*}
       \varphi(\alpha A+(1-\alpha)C)&= \alpha \varphi(A) +(1-\alpha) \varphi(C)\\
                                    &\le \alpha \varphi(B)+ (1-\alpha)\varphi(C)\\
                                    &= \varphi(\alpha B+(1-\alpha)C)
   .\end{align*} 
   In the first and third lines we have used the fact that $\varphi$ is linear on
   the complement of $E$. The second line follows from the fact that $A
   \preceq_{\varphi}
   B$. The preceeding equations show \[
       \alpha A+(1-\alpha)C \preceq_{\varphi} \alpha B+(1-\alpha)C
   .\] 
   To prove that $\preceq_{\varphi}$ is not convex, we show that
   there exists a positive number $\alpha$ less than one and
   distributions $A,B$ and $C$ such that $(\alpha
   A+(1-\alpha)C \preceq_{\varphi} \alpha B+(1-\alpha)C)$ and $\text{not}(A
   \preceq_{\varphi} B)$. By assumption $u$ is
   non-constant on $\Omega\setminus E$ and so there are trajectories
   $\omega_1,\omega_2$ contained in the complement of $E$ such that
   $u(\omega_1)<u(\omega_2)$. Let $\delta(\omega_1),
   \delta(\omega_2)$ be the Dirac distributions concentrated on $\omega_1$ and $\omega_2$,
   respectively. For any trajectory $\omega_E\in E$ and positive number $\alpha$
   less than one, the performance of both $\alpha \delta(\omega_1)+(1-\alpha)\delta(\omega_E)$
   and $\alpha \delta(\omega_2)+ (1-\alpha)\delta(\omega_E)$ are equal to
   $-\beta e^{1-\alpha}$. Hence, \[
       \alpha \delta(\omega_1)+(1-\alpha)\delta(\omega_E)\sim_{\varphi} \alpha
       \delta(\omega_2)+ (1-\alpha)\delta(\omega_E)
   ,\] 
   but $\text{not}(\delta(\omega_1)\sim_{\varphi} \delta(\omega_2))$. So
   $\preceq_{\varphi}$ is not
   convex. Lastly, it is easy to check that for any non-negative number
   $\alpha$ less than or equal to one,
   \begin{align*}
       \text{either}&\quad \varphi(\alpha \delta(\omega_E)+(1-\alpha)\delta(\omega_2))<
       \varphi(\delta(\omega_1))\\
       \text{or}&\quad \varphi(\alpha \delta(\omega_E)+(1-\alpha)\delta(\omega_2))>
       \varphi(\delta(\omega_1))
   .\end{align*}
   Therefore, $\preceq_{\varphi}$ does not satisfy interpolation.
\end{proof}

\end{example}
Totality, transitivity, convexity and interpolation are the axioms
of von Neumann and Morgenstern's seminal result~\citep{vonneumann1947}. We state their theorem using
our notation. In the general case $\Omega$ may be replaced with any finite set.

\begin{theorem}[von Neumann-Morgenstern]
    \label{thm:vnm}
    A binary relation $\preceq$ on the set of distributions over $\Omega$ is a
    total convex preorder satisfying interpolation if and
    only if there is a reward function $r: \mathcal{H}\to \mathbb{R}$ that expresses
    $\preceq$.
    Furthermore, the function $u_r:\Omega\to \mathbb{R}$ given by $u_r(\omega)=
    \sum_{t=0}^{T}r(\xi_{0:t}(\omega))$ is unique up to positive affine transformations.
\end{theorem}
\begin{proof} See~\cite{vonneumann1947}.
\end{proof}

Theorem~\ref{thm:vnm} is a critical result for the Direct Preference Process. On 
one hand, it highlights the backdrop assumptions that are made in the PbRL 
literature~\citep{Chatterji_Neurips2021,Kong2022_provably,Saha2023_dueling,Xu2020_pbrl,Zhu2023_principled_rl}
which assume that preferences are derived from an underlying reward
function. On the other hand, it will allow us to illustrate structural properties and concrete examples of decision problems that have optimal policies \textit{in the absence of
expected reward}.


\section{Conditions for Optimal Policies}
\label{sec:optimal-policies}

Without any assumptions on the goals of a Direct Preference Process, optimal
policies may not exist, making it impossible to proceed with a meaningful theory of preference-based learning. Therefore, in this section we address the
following question: 

    \textit{\textbf{Q1:} Given a Direct Preference Process
        $(\mathcal{O},\mathcal{A},T,e,\preceq)$, what conditions on $\preceq$ are
        sufficient to guarantee the existence of a $\preceq$-optimal policy?}

Our main result of this section, Theorem~\ref{thm:main}, concludes that
(Q1) is satisfied whenever the restriction of $\preceq$ onto the set of
distributions over attainable trajectories is a total, consistent preorder.
One might have hoped that ``rational'' preferences, given by total preorders, would
have been sufficient to guarantee that optimal policies exist. The next
proposition shows that this is not the case. 
\begin{proposition}
    \label{prop:rational-no-solution}
    There is an agent-environment interface $(\mathcal{O},\mathcal{A},T)$,
    environment $e$ and a
    total preorder $\preceq$ on the set of distributions over  $\Omega$ such that the Direct Preference Process 
    $(\mathcal{O},\mathcal{A},T,e,\preceq)$ has no optimal policy.
\end{proposition}
\begin{proof}
Let
$\mathcal{O}= \{o^{0},o^{1}\}$,
$\mathcal{A}= \{a^{0},a^{1}\}$ and $T=2$. To keep notation light, define $h_1^{0}=(o^0, a^0, o^0)$. Suppose that an environment $e$ starts in $o^{0}$ and that for
each history $h_t$, action $a$ and observation $o$, $\rho(o\vert h_t,a)= 1 /2$. For each trajectory
$\omega$, let  $u(\omega; a^{1})$ be the number of times that action $a^{1}$ occurs in
$\omega$. We define the performance $\varphi:\text{Dist}(\Omega)\to \mathbb{R}$ as:
\begin{equation}
    \label{eq:metric}
   \varphi(A)= \begin{cases}
       \sum_{\omega\in \Omega}A(\omega) u(\omega;a^{1})&\text{supp}(A)\subseteq \text{Cyl}(h_1^{0})\\
       - \sum_{\omega\in \Omega}A(\omega) u(\omega; a^{1})&\text{else,}
   \end{cases}
\end{equation} 
where $\text{Cyl}(h_1^{0})$ is the subset of $\Omega$ consisting of every
trajectory that begins with $h_1^{0}$. The performance $\varphi$ measures the
expected number of times that action $a^{1}$ occurs in a given distribution. It may be
``bad'' or ``good'' for $a^{1}$ to occur under a distribution $A$, depending on
whether or not $A$ is supported by $\text{Cyl}(h_1^{0})$. The
performance induces a total preorder $\preceq_{\varphi}$ on $\text{Dist}(\Omega)$
defined for any two distributions $A$ and $B$ over $\Omega$ as: \[
    A \preceq_{\varphi} B \iff \varphi(A)\le \varphi(B)
.\] 
For contradiction, assume there
is an optimal policy $\pi^{\star}$ for the Direct Preference Process $(\mathcal{O},
\mathcal{A},T,e, \preceq_{\varphi})$. The distribution $D^{\pi^{\star}}((o^{0}))$
minimizes the expected number of times that $a^{1}$ occurs since for any policy $\pi$,
$D^{\pi}((o^{0}))$ is not supported by $\text{Cyl}(h_1^{0})$. In particular, 
$\pi^{\star}$ must select action $a^{0}$ in history $h_1^{0}$. Let
$\pi'$ be a policy that selects action $a^{1}$ in history $h_1^{0}$. Then
$\varphi(D^{\pi^{\star}}(h_1^{0}))=0$ and
$\varphi(D^{\pi'}(h_1^{0}))=1$. So $D^{\pi^{\star}}(h_1^{0})\prec_{\varphi} D^{\pi'}(h_1^{0})$, contradicting the assumption that
$\pi^{\star}$ is optimal.
\end{proof}

Notice that the relation $\preceq_{\varphi}$ defined in the proof above does not satisfy consistency. To see this, consider $\omega_1= (o^0, a^0, o^0, a^1, o^0)$, $\omega_2= (o^0, a^0, o^0, a^0, o^0)$ and $\omega_3= (o^1, a^0, o^0, a^0, o^0)$. For each $i\in \{1,2,3\}$, let $\delta(\omega_i)$ be the Dirac distribution concentrated at $\omega_i$. Then $\delta(\omega_1)\succ_{\varphi} \delta(\omega_2)$ but for any positive number $\alpha$ less than one,
\[
\alpha \delta(\omega_1) +(1-\alpha)\delta(\omega_3) \prec_{\varphi} \alpha \delta(\omega_2)+(1-\alpha) \delta(\omega_3)
.
\]
If, however, the goals of a Direct Preference Process also satisfy consistency,
then we have the following result -- extending far beyond (Q1) -- that characterizes optimal
policies with a series of recursive relations.

\begin{theorem}
    \label{thm:main} Let $(\mathcal{O},\mathcal{A},T,e, \preceq)$ be a Direct
    Preference Process. Whenever the restriction of $\preceq$ onto
    $\text{Dist}(\Omega^{e})$ is a total, consistent preorder:
    \begin{enumerate}
        \item There is a deterministic $\preceq$-optimal policy.
        \item If a policy $\pi$ satisfies the following relation for each attainable history
            $h_t$ of length less than $T$ and action $a$, 
            \begin{equation}
                \label{eq:crit_1}
       D^\pi(h_t)\succeq D^{\pi}(h_t\cdot a)
   ,\end{equation}
   then $\pi$ is a $\preceq$-optimal policy.
    \end{enumerate}
\end{theorem}

We prove Theorem~\ref{thm:main} with the following lemma.
\begin{lemma}
\label{lemma:stat1} Let $\alpha_1,\ldots, \alpha_n$ be $n$ non-negative numbers that sum to one and $A_1,\ldots, A_n$, $B_1,\ldots, B_n$ be distributions over a finite set $X$.
If $\preceq$ is a total consistent preorder on $\text{Dist}(X)$ and $A_i\preceq B_i$ for each positive integer $i$ less than $n$, then
\begin{equation*}
\sum_{i=1}^{n} \alpha_i A_i\preceq \sum_{i=1}^{n} \alpha_i B_i
.\end{equation*} 
\end{lemma}
\begin{proof}
    This follows from a simple induction argument, provided in Appendix~\ref{ap:lem-16}.
\end{proof}

 \begin{proof} (\textbf{of Theorem~\ref{thm:main}}) We start by proving (ii). Let $\pi$ be a policy such that for each attainable history $h_t$ of length less than $T$
     and each action $a$, $D^{\pi}(h_t) \succeq D^{\pi}(h_t\cdot a)$. We show by induction
     that the following statement holds for each non-negative integer $t$ less than or equal to $T$:

   \textit{ (P) For each attainable history $h_t$ of length
    $t$ and for any policy
    $\pi'$, $D^{\pi'}(h_t)\preceq D^{\pi}(h_t)$.}

 (P) holds when $t$ is
 equal to $T$ since for any policy $\pi'$ and history $h_T$, both $D^{\pi}(h_T)$ and
 $D^{\pi'}(h_T)$ are equal to the Dirac distribution concentrated on $h_T$. Consider now a time $t$ less than $T$ and
 assume that (P) holds at time $t+1$. Let $h_t$ be an attainable history of length $t$. Since
 the restriction of $\preceq$ onto $\text{Dist}(\Omega^e)$ a total preorder and $\mathcal{A}$ is finite, there is an action $a_{\pi}^{\star}(h_t)$ such
 that $D^{\pi}(h_t\cdot a_{\pi}^{\star}(h_t))$ is a least upper bound, with respect
 to $\preceq$, for the set
 $\{D^{\pi}(h_t\cdot  a):\; a\in A\}$. Therefore, for any policy $\pi'$,
 \begin{align*}
    D^{\pi'}(h_t)&= \sum_{a\in A} \pi'(a \vert h_t)\sum_{o\in \mathcal{O}} \rho(o
    \vert h_t,a) D^{\pi'}(h_t\cdot ( a,o))\\
                &\preceq \sum_{a\in A} \pi'(a \vert h_t) \sum_{o\in \mathcal{O}}
                \rho(o \vert
                h_t,a)D^{\pi}(h_t\cdot (a,o))\\
                &= \sum_{a\in A} \pi'(a \vert h_t) D^{\pi}(h_t\cdot a)\\
                &\preceq D^{\pi}(h_t\cdot a^{\star}_{\pi}(h_t))\\
                &\preceq D^{\pi}(h_t)
 .\end{align*}
 The first and third lines follow from~(\ref{eq:dpi}). The second and fourth lines follow from Lemma~\ref{lemma:stat1} and the assumption that (P) holds at time $t+1$. The last line follows from the initial assumption that for each history $h_t$ and action
 $a$, $D^{\pi}(h_t\cdot a)\preceq D^{\pi}(h_t)$. This concludes the induction step
 and completes the proof of (ii).

 To show (i), we construct deterministic optimal policy that satisfies (ii) as
 follows: for each $t= T-1,\ldots, 0$, and history $h_t$, let $\pi$ select an action
 $a_{\pi}^{\star}(h_t)$ so that $D^{\pi}(h_t\cdot  a_{\pi}^{\star}(h_t))$ is a
 least upper bound for the set $\{D^{\pi}(h_t\cdot a):\; a\in A\}$. Then $\pi$ satisfies the condition in part (ii) of the theorem.
 \end{proof}

Paired with vNM's Expected Utility Theorem, Theorem~\ref{thm:main} has two interesting
implications. First, it is possible for agents to solve preference-based
learning problems even when the objectives cannot be expressed by the expected reward criterion.
\begin{corollary}
    Let $(\mathcal{O},\mathcal{A},T,e, \preceq)$ be a Direct
    Preference Process. If the restriction of $\preceq$ onto $\text{Dist}(\Omega^{e})$ is a total consistent preorder that is either not convex or does not satisfy interpolation, then an optimal policy exists but $\preceq$ cannot be expressed by the expected reward criterion.
\end{corollary}
\begin{proof}
    Immediate from Theorems~\ref{thm:vnm} and~\ref{thm:main}.
\end{proof}

This is the case in Example~\ref{ex:consistent}
as well as our next example.
\begin{example}[Tie-breaking Criterion]
    \label{ex:nointerpolation}
    Let $u_1$ and $u_2$ be two real-valued functions on $\Omega$. For each $i\in
    \{1,2\}$ and distribution $A$ over $\Omega$, let $u_i(A)$ denote the expected value
   of $u_i$ under $A$. Define the relation $\preceq$ on $\text{Dist}(\Omega)$ 
   according to the following two rules:
   \begin{itemize}
       \item[R1:] For any two distributions $A$ and $B$ over $\Omega$, if $u_1(A)<u_1(B)$ then $A\prec
           B$.
       \item[R2:] For any two distributions $A$ and $B$ over $\Omega$, if $u_1(A)=u_1(B)$ then $(A
           \preceq B \iff u_2(A)\le u_2(B))$.
   \end{itemize}
   Under these rules, $u_2$ acts as a ``tie-breaking criterion'' when two
   distributions achieve the same performance on $u_1$.
   Assuming that $u_1$ is non-constant and there are distributions $A$ and $B$ such that
   $u_1(A)=u_1(B)$ and $u_2(A)\neq u_2(B)$, the relation $\preceq$ defined by
   (R1) and (R2) is a total, convex
   preorder that does not satisfy interpolation.

\begin{proof}
    It is easy to check that $\preceq$ is a total convex preorder.
    By assumption there are distributions $A$ and $B$ over $\Omega$ such that $u_1(A)= u_1(B)$ and
$u_2(A)< u_2(B)$. Since $u_1$ is non-constant there is a
distribution $C$ such that $u_1(C)\neq u_1(A)$. If $\alpha$ is equal to
one then $\alpha A+(1-\alpha)C\prec B$ because $A\prec B$. If  $\alpha$ is non-negative and less than one,\[
    u_1(\alpha A+(1-\alpha)C)\neq u_1(B)
,\] 
and so $\text{not}(\alpha A+(1-\alpha)C\sim B)$. Hence, $\preceq$ does not satisfy interpolation.
\end{proof}

\end{example}
A second implication of Theorem~\ref{thm:main} is that the Bellman Optimality Equations~\citep{Bellman1957}
that characterize optimal policies in RL are a consequence of a more general result
that holds for total consistent preorders. We obtain Bellman's equations as a
consequence of the second part of Theorem~\ref{thm:main}.
\begin{corollary}
    \label{cor:bellman}
    Let $(\mathcal{O},\mathcal{A},T,e, \preceq)$ be a Direct Preference Process.
    Whenever $\preceq$ is expressed by a reward function $r:\mathcal{H}\to
    \mathbb{R}$, a policy 
    $\pi$ is optimal if and only if it satisfies the following equation for each attainable history $h_t$ of length less than $T$:
    \[
        V_{\pi}(h_t;r)= \max_{a\in \mathcal{A}} \left( r(h_t)+ \sum_{o\in
            \mathcal{O}}\rho(o \vert h_t,a)
        V_{\pi}(h_t\cdot (a,o);r) \right)
   .\] 
\end{corollary}
\begin{proof}
    Immediate from Theorems~\ref{thm:rl-opt} and~\ref{thm:main}.
\end{proof}

In light of Proposition~\ref{prop:rational-no-solution} and Theorem~\ref{thm:main},
a minimal and robust assumption to further develop a theory of LfPF is the following. 

\begin{assumption}\label{a:1} 
The restriction of $\preceq$ onto the set of distributions
over attainable trajectories is a total consistent preorder.
\end{assumption}

\subsection{Optimal Action Sets}
\label{sec:opt-action-sets}
A third consequence of Theorem~\ref{thm:main} is that all optimal
policies in a Direct Preference Process satisfying Assumption~\ref{a:1} are characterized by
a set of ``optimal actions'' for each attainable history. This gives rise to a
useful characterization of optimal policies which we will use in the next section.
\begin{definition}\label{def:Api}
    Let $(\mathcal{O},\mathcal{A},T,e,\preceq)$ be a Direct Preference Process. For
    each policy $\pi$ and history $h_t$ of length less than $T$, define $\mathcal{A}_{\pi}^{\star}(h_t)$ as the set of actions $a$ for which
$D^{\pi}(h_t\cdot a)$ is a least upper bound for the set $\{D^{\pi}(h_t\cdot a'):\; a'\in
\mathcal{A}\}$. More precisely, $\mathcal{A}_{\pi}^{\star}(h_t)$ consists of every
action $a$ for which the following holds:
\begin{equation}
 \forall a'\in \mathcal{A}, \quad D^{\pi}(h_t\cdot a) \succeq
    D^{\pi}(h_t\cdot a')
.\end{equation} 
\end{definition}
\begin{lemma}\label{lem:opt-actions} 
Let $(\mathcal{O},\mathcal{A},T,e,\preceq)$ be a Direct Preference Process that
satisfies Assumption~\ref{a:1}. For any two optimal policies $\pi$ and $\pi'$ and attainable history $h_t$ 
    of length less than $T$, $\mathcal{A}_{\pi}^{\star}(h_t)$ is equal to
    $\mathcal{A}_{\pi'}^{\star}(h_t)$.
\end{lemma}

\begin{proof}
Let $\pi$ and $\pi'$ be two optimal policies. For any attainable history $h_t$ of length less
than $T$, action $a$ and observation $o$, if $h_t\cdot (a,o)$ is attainable in $e$
then the distributions $D^{\pi}(h_t\cdot
(a,o))$ and $D^{\pi'}(h_t\cdot (a,o))$ are $\preceq$-equivalent. Thus, by consistency we have
that for any attainable history $h_t$ of length less than $T$ and action $a$ the distributions $D^{\pi}(h_t\cdot a)$ and
$D^{\pi'}(h_t\cdot a)$ are $\preceq$-equivalent, and the result follows immediately.
\end{proof}

In view of this lemma, we drop the dependence of $\mathcal{A}_{\pi}^{\star}(h_t)$ on
$\pi$ when $\pi$ is an optimal policy. We call $\mathcal{A}^{\star}(h_t)$ the
\textit{optimal action set for $h_t$}. 

\begin{corollary}
    \label{cor:opt-criteria}
    Let $(\mathcal{O},\mathcal{A},T,e, \preceq)$ be a Direct Preference Process that
    satisfies Assumption~\ref{a:1}. A policy $\pi$ is optimal if and only if for each attainable history $h_t$ of length less
    than $T$, $\pi(\cdot  \vert h_t)$ is supported by $\mathcal{A}^{\star}(h_t)$.
\end{corollary}

\begin{proof}
By the second part of of Theorem~\ref{thm:main}, if $\pi$ is an optimal policy then for each attainable history $h_t$ of length less than $T$, $\pi(\cdot
\vert h_t)$ is supported by $\mathcal{A}_{\pi}^{\star}(h_t)$, which is equal to
$\mathcal{A}^{\star}(h_t)$ by Lemma~\ref{lem:opt-actions}. Conversely,
if for each attainable history $h_t$ of length less than $T$, $\pi(\cdot  \vert
h_t)$ is supported by $\mathcal{A}^{\star}(h_t)$ then there is an optimal policy
$\pi'$ such that for each attainable history $h_t$, $D^{\pi}(h_t)$ is
$\preceq$-equivalent to $D^{\pi'}(h_t)$. Therefore $\pi$ is an optimal policy.
\end{proof}

\section{Optimal Feature-based Policies}
\label{sec:bounded-optimal-policies}
In real-world applications, agents face computational constraints and make
decisions based on a limited set of relevant information, known as ``features'', derived from their history. 
Hence, the concept of an optimal ``feature-based'' policy is crucial for a theory
of preference-based learning.

The main question of this section (Q2) concerns a
computationally-constrained agent which can only  access a finite set of features,
denoted as $\mathcal{X}$. A \textit{feature map} $\phi: \mathcal{H} \to \mathcal{X}$ determines the feature
retained from each history.

\begin{example}
\label{ex:traj-segments}
Given a positive integer $k<T$, the feature of each history
can be the sub-string  of the most recent $k$ observations and $k-1$ actions. In this case,
$\mathcal{X}=
\bigcup_{l=0}^{k-1} \mathcal{H}_l$ and the feature map $\phi$ is defined for each
    history $h_t=(o_0,a_0,\ldots, a_{t-1},o_t)$ as: 
\[
    \phi(h_t):= \begin{cases}
        h_t&t< k\\
        (o_{t-k+1}, a_{t-k+1},\ldots, a_{t-1},o_t)&t\ge k.
    \end{cases}
\] 
\end{example}

We define a feature-based policy as one whose action selection in each history $h_t$
depends only on $\phi(h_t)$.

\begin{definition}\label{defn:pphi}
    Given a feature map $\phi$, $\pi$ is a \textbf{feature-based policy} if for each pair of
$t$-histories $h_t,h_t'$ of length less than  $T$,
\begin{equation}
    \label{eq:pphi}
    (\phi(h_t)= \phi(h_t')) \implies (\pi(\cdot  \vert h_t)= \pi(\cdot  \vert
    h_t'))
.\end{equation}
We define $\Pi^{\phi}$ as the set of feature-based policies.
\end{definition}
In Example~\ref{ex:traj-segments}, $\Pi^{\phi}$ is the set of policies whose action
selection in each history depends only on the history through its final $k$
observations and $k-1$ actions. The core objective of this section is to address (Q2):

\textit{\textbf{Q2:} Given a Direct Preference Process that satisfies
    Assumption~\ref{a:1}, what
        conditions does a feature map $\phi$ need to satisfy in order to guarantee that $\Pi^{\phi}$ contains an optimal
        policy?}

The optimal action sets described in Section~\ref{sec:opt-action-sets} provide a
necessary and sufficient condition to address (Q2).
\begin{proposition}
    \label{prop:pi-phi}
    If a Direct Preference Process $(\mathcal{O},
    \mathcal{A},T,e,\preceq)$ satisfies Assumption~\ref{a:1} then for
    any feature map $\phi$,
    $\Pi^{\phi}$ contains an optimal policy if and only if for each attainable history $h_t$ of
    length less than $T$,
    \begin{equation}
        \label{eq:opt}
        \bigcap_{h_t'\in \phi^{-1}(\phi(h_t))\cap
        \mathcal{H}_t^{e}}\mathcal{A}^{\star}(h_t') \neq \emptyset
    .\end{equation} 
\end{proposition}
\begin{proof}
    By part (iii) of Corollary~\ref{cor:opt-criteria} and the definition of
    $\Pi^{\phi}$, if $\pi$ is an optimal feature-based policy then for each attainable history
    $h_t$ of length less than $T$ and attainable $t$-history $h_t'$ in the preimage
    of $\phi(h_t)$, the distribution $\pi(\cdot  \vert h_t)$ must be supported by
    $\mathcal{A}^{\star}(h_t')$. Hence, the intersection in (\ref{eq:opt})
    must be non-empty. To show the converse, construct an optimal policy as follows:
for each attainable history $h_t$ of length less than $T$, let $\pi$ select an action in the
intersection of $\bigcap_{h_t'\in \phi^{-1}(\phi(h_t))\cap \mathcal{H}_t^{e}} \mathcal{A}^{\star}(h_t')$. Then
$\pi$ is optimal by Corollary~\ref{cor:opt-criteria} and is contained in
$\Pi^{\phi}$ by construction. 
\end{proof}

Roughly speaking, Proposition~\ref{prop:pi-phi} shows that when the goals of a Direct
Preference Process satisfy Assumption~\ref{a:1}, an optimal
feature-based policy exists if, and only if, for each attainable $t$-history
$h_t$, there is an action that is simultaneously optimal for every
attainable $t$-history in the preimage of $\phi(h_t)$.

\subsection{Embedded Preferences}
\label{sec:pref-embeddings}
Although Proposition~\ref{prop:pi-phi} gives both a necessary and sufficient condition
that answers (Q2), the condition is rather generic and it is hard to check whether
or not a system satisfies it. In this section we present Theorem~\ref{thm:markov},
which provides verifiable conditions to
answer (Q2). While not necessary, these conditions offer practical ways to ensure
that optimal feature-based policies exist. 
They rely on the following notion of weighted averages.

\begin{definition}[$(\phi,\gamma)$-Frequency]\label{def:discounted-freq} 
    Let $\phi:\mathcal{H}\to \mathcal{X}$ be a feature map and
    $(\gamma_t)_{t=1}^{T-1}$ be a sequence of non-negative numbers that are not all
    zero. Given two non-negative integers $t_1$ and $t_2$ such that $t_1\le t_2\le
    T$ and for which $\sum_{t=t_1}^{t_2-1}\gamma_t$ is non-zero, define the function
$f^{(\phi,\gamma)}_{t_1:t_2}:
\mathcal{X}\times
\mathcal{A} \times \text{Dist}(\Omega)\to [0,1]$ as 
\begin{equation}
    \label{eq:freq}
    f_{t_1:t_2}^{(\phi,\gamma)}(x,a \vert D):=
    \frac{1}{\sum_{t=t_1}^{t_2-1} \gamma_t}
    \sum_{t=t_1}^{t_2-1}\gamma_t \mathbb{P}_D((X_t,A_t)=(x,a)),
\end{equation} 
where $\mathbb{P}_{D}((X_t,A_t)=(x,a))$ is the probability that the feature-action
pair $(x,a)$ is visited at time $t$ under distribution $D$.
We say that $f_{t_1:t_2}^{(\phi,\gamma)}(x,a
\vert D)$ is the \textbf{$(\phi,\gamma)$-frequency of $(x,a)$ in distribution
$D$ in between $t_1$ and $t_2$}.
When $\sum_{t=t_1}^{t_2-1}\gamma_t=0$ we define
$f_{t_1:t_2}^{(\phi,\gamma)}(x,a \vert D)=0$. We abbreviate
$f_{0:T}^{(\phi,\gamma)}(x,a \vert D)$ to $f^{(\phi,\gamma)}(x,a \vert D)$.
\end{definition}

\textbf{Interpretation of $\gamma$}. The
$(\phi,\gamma)$-frequency is a weighted measure of how often
each feature-action pair is visited in a given distribution over $\Omega$. The
weights $(\gamma_t)_{t=1}^{T-1}$ measure the
importance of the time at which feature-action pairs are visited. For
instance, if $\gamma_t$ is equal to one for each
time $t$, the distribution $f^{(\phi,\gamma)}(\cdot  \vert D)$ measures the relative frequency of
feature-action pairs visited under $D$. If $\gamma_t= \alpha^t$ for some
positive number $\alpha$ less than one, $f^{(\phi,\gamma)}(\cdot  \vert D)$ measures the
$\alpha$-discounted frequency of feature-action pairs visited under $D$.

\begin{lemma}\label{lem:freq}
    When $\sum_{t=t_1}^{t_2-1}\gamma_t$ is non-zero the function $(x,a)\mapsto
    f_{t_1:t_2}^{(\phi,\gamma)}(x,a \vert D)$ defines a probability distribution over the
    set of feature-action pairs, which we denote by $f_{t_1:t_2}^{(\phi,\gamma)}(\cdot  \vert D)$. 
\end{lemma}
\begin{proof}
    Immediate from Definition~\ref{def:discounted-freq}.
\end{proof}

Using the $(\phi,\gamma)$-frequency map, we are now able to describe goals that ``only
depend'' on the weighted frequency of feature action pairs. 
\begin{definition}\label{defn:embed}
    Let $(\mathcal{O},\mathcal{A},T,e,\preceq)$ be a Direct Preference Process and
    let $\preceq_{\circ}$ be a binary relation on the set of distributions over
    $\mathcal{X}\times \mathcal{A}$. We
    say that $\preceq$ \textbf{preserves and reflects $\preceq_{\circ}$ via
    $(\phi,\gamma)$-frequency} if for any two distributions $A$ and $B$ over $\Omega$, 
   \begin{equation}
   \label{eq:embed}
   A \preceq B \iff f^{(\phi,\gamma)}(\cdot  \vert A) \preceq_{\circ}
   f^{(\phi,\gamma)}(\cdot  \vert B)
   .\end{equation} 
\end{definition}

We say that $\preceq$ \textbf{embeds into $\preceq_{\circ}$ via
$(\phi,\gamma)$-frequency} whenever $\preceq$ preserves and reflects
$\preceq_{\circ}$ via $(\phi,\gamma)$-frequency, despite the fact that the map
$D\mapsto f^{(\phi,\gamma)}(\cdot  \vert D)$ is neither injective nor
surjective, and thus not an order embedding. 

The next two examples show how the $(\phi,\gamma)$-frequency embedding is useful when preferences are given between observation-action pairs~\citep{Stiennon_2020learning} or trajectory segments~\citep{Christiano2017_RLHF,Kim2022_preference}. In these situations, we can use the $(\phi,\gamma)$-frequency map to define a preference relation on $\text{Dist}(\Omega)$ from the preference data.

\begin{example}[Preferences over Observation-Action Pairs]
Consider $v_1, v_2:\mathcal{O}\times\mathcal{A}\to [0,1]$ and define the relation $\preceq_{\circ}$ on $\text{Dist}(\mathcal{O}\times \mathcal{A})$ according to the Tie-breaking Criterion from Example 16. If $\phi$ maps each history to its most recent observation, then $\mathcal{X}=\mathcal{O}$ and $\preceq_{\circ}$ is an ordering on $\text{Dist}(\mathcal{X}\times\mathcal{A})$. For a sequence of positive weights $(\gamma_t)_{t=0}^{T-1}$, we can define a relation $\preceq$ on $\text{Dist}(\Omega)$ via Equation~\ref{eq:embed}. In this case, the preferences given by $\preceq$ depend only on the weighted frequency of observation-action pairs. In particular, distributions $A$ and $B$ over $\Omega$ are $\preceq$-equivalent if they visit all observation action pairs with the same weighted frequency.
\end{example}

\begin{example}[Preferences over Trajectory Segments] Let $k$ be a positive integer less than $T$.
    Suppose that preference data is available for histories of length up to $k$, giving rise to a binary relation
    $\preceq_{\circ}$ on the set of distributions
    over $\bigcup_{l=0}^k (\mathcal{H}_l\times \mathcal{A})$. With the feature map
    defined in Example~\ref{ex:traj-segments} and a sequence of non-negative numbers
    $\gamma=(\gamma_t)_{t=0}^{T-1}$, the goals of a
    Direct Preference Process can be defined using  $\preceq_{\circ}$ and Equation~\ref{eq:embed}.
\end{example}

Although Definition~\ref{defn:embed} ensures that the learning objectives are fully described by distributions over feature-action pairs, we require
additional assumptions on the transition dynamics of the environment to ensure that
optimal feature-based policies exist.

\begin{definition}A Direct Preference Process
    $(\mathcal{O},\mathcal{A},T,e,\preceq)$ and feature map $\phi$ satisfy the
    \textbf{Markov Feature Assumption} when the following two statements hold for each pair of 
    attainable $t$-histories $h_t,h_t'$ of length less than $T$:
    \begin{itemize}
        \item If $\phi(h_t)=\phi(h_t')$ then for each action $a$, $\rho(\cdot  \vert h_t,a)=\rho(\cdot  \vert h_t',a)$.
        \item If $\phi(h_t)=\phi(h_t')$ then for each action $a$ and
            observation $o$, $\phi(h_t\cdot (a,o))=\phi(h_t'\cdot(a,o))$.
    \end{itemize}
\end{definition}
The latter conveys the notion that if a feature accounts for all the retained
information in each history, then the feature in each history depends on its
sub-histories only through previous features. This is satisfied by the agents
considered in many popular agent
designs~\citep{Lu2023_BitbyBit,mnih2015humanlevel,Osband2016}.
Combined with Definition \ref{defn:embed}, the Markov Feature Assumption
guarantees that optimal feature-based policies exist.

\begin{theorem}
    \label{thm:markov}
    Let $(\mathcal{O}, \mathcal{A},T,e,\preceq)$ be a Direct Preference Process and
    $\preceq_{\circ}$ a total consistent preorder on
    $\text{Dist}(\mathcal{X}\times \mathcal{A})$ such that  $\preceq$ embeds into $\preceq_{\circ}$ 
    via $(\phi,\gamma)$-frequency.
    \begin{enumerate}
        \item Every policy $\pi$ that satisfies the following recursive relation for each attainable
    history  $h_t$ of length less than $T$ and action $a$ is a $\preceq$-optimal policy: 
    \begin{equation}
    \label{eq:future}
    f_{t:T}^{(\phi,\gamma)}(\cdot  \vert D^{\pi}(h_t)) \succeq_{\circ}
    f_{t:T}^{(\phi,\gamma)}(\cdot  \vert D^{\pi}(h_t\cdot a))
    .\end{equation}

        \item If the Markov Feature Assumption is satisfied then $\Pi^{\phi}$ 
            contains a $\preceq$-optimal policy.
    \end{enumerate}
\end{theorem}

To prove Theorem~\ref{thm:markov} we introduce some more notation. First, let
$\Gamma_{t_1:t_2}=\sum_{t=t_1}^{t_2-1} \gamma_t$. Second, given a set of
distributions $\mathcal{D}= \{D(a'):\; a'\in \mathcal{A}\}$ parameterized by the set
of actions $\mathcal{A}$, we define $\text{arglub}_{\preceq}\mathcal{D}$ as the set of
actions $a$ for  which $D(a)$ is a least upper bound for $\mathcal{D}$ (with respect
to $\preceq$). Note that such least upper bounds exist when $\mathcal{A}$ is finite and the restriction of $\preceq$ onto $\mathcal{D}$ is a total preorder. Lastly, for
each attainable history $h_t$ of length less than $T$ and policy $\pi$, when $\Gamma_{t_1:T}$ is
nonzero we define
$\mathcal{F}_{\pi}^{\star}(h_t)$ as the set of actions which lead to the best
distribution over future feature-action pairs, 
\begin{equation}
    \label{eq:opt-future-dist}
    \mathcal{F}_{\pi}^{\star}(h_t)= \text{arglub}_{\preceq_{\circ}} \{f_{t:T}^{(\phi,\gamma)}(\cdot  \vert
    D^{\pi}(h_t\cdot a)):\; a\in \mathcal{A}\}
.\end{equation} 
When $\Gamma_{t_1:T}$ is equal to zero we let $\mathcal{F}_{\pi}^{\star}(h_t)=\mathcal{A}$. Notice here that we have suppressed the dependence of
$\mathcal{F}_{\pi}^{\star}(h_t)$ on $\phi$ and $\gamma$ to keep notation light. However, it is important to remember that $\mathcal{F}_{\pi}^{\star}(h_t)$ depends both
on $\phi$ and $\gamma$. 

\begin{proof}\textbf{(of Theorem~\ref{thm:markov})}
Theorem~\ref{thm:markov} follows quickly from the fact that for any attainable $t$-history $h_t$ of length less than $T$, both the following
    statements are true: 
    \begin{itemize}
        \item[(F1)] For each policy $\pi$, $\mathcal{F}_{\pi}^{\star}(h_t)$ is a non-empty subset of
            $\mathcal{A}_{\pi}^{\star}(h_t)$.
        \item[(F2)] If the Markov Feature Assumption is satisfied then for any optimal policy
            $\pi$ and attainable $t$-history $h_t'$ in the preimage of $\phi(h_t)$,
            $\mathcal{F}_{\pi}^{\star}(h_t)$ is equal to
            $\mathcal{F}_{\pi}^{\star}(h_t')$. 
    \end{itemize}

The proofs of (F1) and (F2) are left to Appendix~\ref{ap:thm-32}. At a high level, (F1) follows from the definitions of $(\phi,\gamma)$-frequency (Definition~\ref{def:discounted-freq}) and of $\mathcal{F}_{\pi}^{\star}(h_t)$ in~(\ref{eq:opt-future-dist}). (F2) holds by an induction argument similar to the one given in Theorem~\ref{thm:main}.

    \textit{Proof of (i)}. Since $\preceq_{\circ}$ is a total consistent preorder, Equation~\ref{eq:future} holds if
and only if $f_{t:T}^{(\phi,\gamma)}(\cdot  \vert D^{\pi}(h_t))$ is supported by
$\mathcal{F}_{\pi}^{\star}(h_t)$. Therefore, by (F1) if
Equation~\ref{eq:future} holds
for each attainable history $h_t$ of length less than $T$ then $\pi(\cdot  \vert
h_t)$ is supported by $\mathcal{A}_{\pi}^{\star}(h_t)$ for each attainable history $h_t$ of length less than $T$. It is easy to check that the relation $\preceq$ is a total consistent preorder
since it embeds into $\preceq_{\circ}$ via $(\phi,\gamma)$-frequency and thus we may use the third criterion for
optimality in Corollary~\ref{cor:opt-criteria} to see that $\pi$ is an optimal policy.

\textit{Proof of (ii)}. If both (F1) and (F2) are true then for any optimal policy $\pi$
and attainable
history $h_t$ of length less than $T$, 
\begin{equation}
    \label{eq:s}
    F_{\pi}^{\star}(h_t) \subseteq \bigcap_{h_t'\in \phi^{-1}(\phi(h_t))\cap
    \mathcal{H}_t^{e}} \mathcal{A}^{\star}(h_t)
.\end{equation} 
In particular, for each attainable history $h_t$ of length less than $T$ the intersection on the right hand side of
(\ref{eq:s}) is non-empty for each attainable history
$h_t$ of length less than $T$. So $\Pi^{\phi}$ contains an optimal policy by
Proposition~\ref{prop:pi-phi}. 
\end{proof}

The first part of Theorem~\ref{thm:markov} shows that if the goals of a Direct
Preference Process are preserved and reflected by a total consistent preorder on the set of distributions over
feature-action pairs, then a policy is optimal whenever it achieves the most
desirable distribution over future feature-action pairs in every starting history.
However, without any assumptions on the environment's transition dynamics, a
feature-based policy may not satisfy the conditions in part (i). 
The Markov Feature Assumption is a strong assumption and relaxing these conditions
is an important area for future work. When the Markov Feature Assumption does not
hold, Proposition~\ref{prop:pi-phi} provides an alternative means to
guarantee the existence of feature-based policies.

\subsection{Connection to Markov Rewards}
We introduced the embedding of preferences via $(\phi,\gamma)$-frequency as an abstract
property that might help to specify the goals of a Direct Preference Process. 
Example~\ref{ex:traj-segments} demonstrates the practical utility of this
property when preference data is collected on trajectory segments
rather than full-length trajectories. Nevertheless,  some readers may be
hesitant about its justification. To address these doubts, our final result shows that the $(\phi,\gamma)$-frequency embedding is
implied by any objective defined by Markov rewards.
\begin{theorem}
    \label{thm:markov-reward}
    Let $(\mathcal{O},\mathcal{A},T,e, \preceq)$ be Direct Preference Process, $\phi:
    \mathcal{H}\to \mathcal{X}$ be a feature map and $(\gamma_{t})_{t=1}^{T-1}$ be a
    sequence of non-negative numbers that are not all zero. The following two statements are
    equivalent:
    \begin{enumerate}
            \item $\preceq$ embeds into a total convex preorder $\preceq_{\circ}$ that satisfies interpolation via
            $(\phi,\gamma)$-frequency.
        \item There is a reward function $r:\mathcal{X}\times \mathcal{A}\to
            \mathbb{R}$ such that for any two distributions $D$ and $D'$ over
            $\Omega$, $D \preceq D'$ if and only if \[
                \mathbb{E}_{D} \left[ \sum_{t=1}^{T-1} \gamma_t r(X_t,A_t) \right] \le \mathbb{E}_{D'}\left[ \sum_{t=1}^{T-1} \gamma_t r(X_t,A_t) \right]
            .\] 
    \end{enumerate}
\end{theorem}

\begin{proof} 
    \textit{(i) $\implies$ (ii)}. By the definition of the $(\phi,\gamma)$-frequency embedding (Definition~\ref{defn:embed}) and the vNM Expected Utility Theorem (Theorem~\ref{thm:vnm}), there is a function $r:\mathcal{X}\times \mathcal{A}\to \mathbb{R}$ such that for any two distributions $D$ and $D'$ over $\Omega$, $D\preceq D'$ if and only if
     \begin{equation*}
    \label{eq:s_2}
    \sum_{x,a}f^{(\phi,\gamma)}(x,a  \vert D)r(x,a)\le\sum_{x,a}f^{(\phi,\gamma)}(x,a \vert D')r(x,a)
    .\end{equation*}
    It is also easy to check that for any distribution $D$ over $\Omega$,
    \begin{equation}
    \label{eq:s_3}
    \sum_{x,a}
f^{(\phi,\gamma)}(x,a \vert D)r(x,a)= \frac{1}{\sum_{t=1}^{T-1} \gamma_t} \mathbb{E}_{D} \left[ \sum_{t=1}^{T-1} \gamma_t r(X_t,A_t) \right]
    .\end{equation}
    So for any two distributions $D$ and $D'$ over $\Omega$, $D\preceq D'$ if and only if
    \begin{align*}
     \frac{1}{\sum_{t=1}^{T-1} \gamma_t} \mathbb{E}_{D}\left[
                    \sum_{t=1}^{T-1} \gamma_t r(X_t,A_t) \right]\le \frac{1}{\sum_{t=1}^{T-1} \gamma_t} \mathbb{E}_{D'}\left[ \sum_{t=1}^{T-1} \gamma_t r(X_t,A_t) \right]
    .\end{align*}
    Statement \textit{(ii)} then follows by multiplying both sides of the above inequality by $\sum_{t=1}^{T-1}\gamma_t$.

    \textit{(ii) $\implies$ (i)}. Define a binary relation $\preceq_{\circ}$ on
    the set of distributions over feature-action pairs for any two
    distributions $\mu_1,\mu_2$ as:
    \begin{equation*}
        \mu_1 \preceq_{\circ} \mu_2 \iff \sum_{x,a}\mu_1(x,a)r(x,a)\le \sum_{x,a}\mu_2(x,a)r(x,a)
    .\end{equation*} 
    By the vNM Expected Utility Theorem, $\preceq_{\circ}$ is a total convex preorder satisfying interpolation. From the assumption in part (ii) and Equation~\ref{eq:s_3}, it follows that for any two distributions distributions $D$ and $D'$ over $\Omega$, $D\preceq D'$ if and only if
    \begin{align*}
       \sum_{x,a} f^{(\phi,\gamma)}(x,a \vert D)r(x,a)\le \sum_{x,a}f^{(\phi,\gamma)}(x,a \vert D')r(x,a)
    .\end{align*}
    Therefore, $\preceq$ must embed into $\preceq_{\circ}$ via $(\phi,\gamma)$-frequency.
\end{proof}

It is interesting to compare Theorem~\ref{thm:markov-reward} with vNM's Expected Utility Theorem. If the goals are
expressed by a feature-action reward function, as opposed to a history-based reward,
then $\preceq$ embeds into an underlying feature-action
preference  via
$(\phi,\gamma)$-frequency. However, the feature-action reward
function which expressed $\preceq$ may not be unique, as multiple feature-action preferences could preserve and
reflect $\preceq$. This result complements previous work on Markov reward expressiveness in finite MDPs~\citep{Abel2021_expressivity,Skalse2023_reward_limitations}.

\section{Conclusion}
\label{sec:conclusion}

We introduced the Direct Preference Process, a model of preference-based learning in
partially-observable, non-Markovian environments. Unlike previous work, we
did not assume that preferences were generated by an underlying reward
function. Instead we used conditions on the ordinal structure of the preferences to
guarantee the existence of optimal policies. 
We showed that it is possible for
an agent to behave optimally with respect to a given set of
preferences even when there is no corresponding reward function that captures the same learning goal. Lastly, we
provided two results to determine when it is possible for a
computationally-constrained agent to behave optimally, as well as a characterization
of goals expressed by Markov rewards.

The Direct Preference Process opens up many interesting avenues for future work.
An extension of this framework for infinite observation and action sets is important for preference-based robotics tasks. For practitioners, it is
interesting to study whether  agents can perform well without learning a reward model. Recent findings~\citep{Rafailov2023direct,An2023_DPPO,Kang2023_oppo}
have shown that this may be the case. Moreover, the notion of
a $(\phi,\gamma)$-frequency embedding could be used to \textit{derive} relevant
features for a preference-based decision problem. Finally, it would be very useful to study the hardness of learning feature maps. We suspect that this may highlight differences between reward-based vs purely preference-based agents.


\subsection*{Acknowledgements}
We thank the reviewers and members of the RL Lab for their helpful feedback. JCC is grateful for the support of an NSERC CGS-M scholarship, PP is grateful for the support of an NSERC research grant and DP is grateful for funding from NSERC and CIFAR.

\bibliography{references}

\appendix

\section{Missing Proofs}

\subsection{Proof of Lemma~\ref{lemma:stat1}}
\label{ap:lem-16}
\textbf{Lemma 16.} 
Let $\alpha_1,\ldots, \alpha_n$ be $n$ non-negative numbers that sum to one and $A_1,\ldots, A_n$, $B_1,\ldots, B_n$ be distributions over a finite set $X$.
If $\preceq$ is a total consistent preorder on $\text{Dist}(X)$ and $A_i\preceq B_i$ for each positive integer $i$ less than $n$, then
\begin{equation*}
\sum_{i=1}^{n} \alpha_i A_i\preceq \sum_{i=1}^{n} \alpha_i B_i
.\end{equation*} 

\begin{proof}
The case when one of the $\alpha_i$'s is equal to one is trivial, so assume
otherwise. Then for each $i$, the number $\overline{\alpha}_i:=1-\alpha_i$ is
positive, the sum $\sum_{j\neq i}\frac{\alpha_j}{\overline{\alpha}_i}$ is equal to $1$ and
for any distributions $C_1,\ldots,
C_n$, the distribution $\sum_{i=1}^{n} \alpha_i C_i$ can be re-written as:
\begin{equation}
\label{ap-eq:lemma-fact1}
\sum_{j=1}^n \alpha_j C_j= \alpha_i C_i + \overline{\alpha}_i \left( \sum_{j\neq i}
\frac{\alpha_j}{\overline{\alpha}_i}C_j\right)
.\end{equation}
To prove the lemma, we will show by induction that for every non-negative integer $k$ less than or equal to
$n$, 
\begin{equation}
\label{ap-eq:ind}
\sum_{i=1}^n A_i \preceq \sum_{i=1}^k \alpha_iB_i+ \sum_{i=k+1}^n \alpha_i A_i
.\end{equation}
The statement in the lemma corresponds to the case when $k$ is equal to $n$.
Equation~\ref{ap-eq:ind} holds when $k$ is equal to zero since $\preceq$ is reflexive.
Assuming now
that~(\ref{ap-eq:ind}) holds for a fixed non-negative integer $k$ less than $n$,
\begin{align*}
   \sum_{i=1}^{n} \alpha_i A_i &\preceq \sum_{i=1}^{k} \alpha_i B_i +
   \sum_{i=k+1}^{n} \alpha_i A_i\\
           &= \alpha_{k+1}A_{k+1}+ \overline{\alpha}_{k+1} \left( \sum_{i=1}^{k} \frac{\alpha_i}{\overline{\alpha}_{k+1}}B_i+ \sum_{i=k+2}^{n} \frac{\alpha_i}{\overline{\alpha}_{k+1}} A_i \right)\\
           &\preceq \alpha_{k+1}B_{k+1}+ \overline{\alpha}_{k+1} \left( \sum_{i=1}^{k} \frac{\alpha_i}{\overline{\alpha}_{k+1}}B_i+ \sum_{i=k+2}^{n} \frac{\alpha_i}{\overline{\alpha}_{k+1}} A_i \right)\\
           &= \sum_{i=1}^{k+1} \alpha_i B_i+ \sum_{i=k+2}^{n} \alpha_i A_i
.\end{align*}
The first line follows from our inductive assumption. The second and fourth lines follow from (\ref{ap-eq:lemma-fact1}). The third line follows from
the consistency of $\preceq$. This concludes the induction step since $\preceq$ is transitive.
\end{proof}

\subsection{Proofs of (F1) and (F2) from Theorem~\ref{thm:markov}}
\label{ap:thm-32}

To complete the proof of Theorem~\ref{thm:markov}, it remains to prove (F1) and (F2).

\textbf{F1.} Let $(\mathcal{O},\mathcal{A},T,e,\preceq)$ be Direct Preference Process and
    $\preceq_{\circ}$ a total consistent preorder on
    $\text{Dist}(\mathcal{X}\times \mathcal{A})$ such that  $\preceq$ embeds into $\preceq_{\circ}$ 
    via $(\phi,\gamma)$-frequency. For each attainable history $h_t$ of length less than $T$ and policy $\pi$, $\mathcal{F}_{\pi}^{\star}(h_t)$ is a non-empty subset of $\mathcal{A}_{\pi}^{\star}(h_t)$.

\begin{proof} $\mathcal{F}_{\pi}^{\star}(h_t)$ is non-empty since $\mathcal{A}$ is
a non-empty finite set and
$\preceq_{\circ}$ is a total preorder. For every distribution $D$ over $\Omega$, the $(\phi,\gamma)$-frequency of
feature-action pairs in between $t_1$ and $t_2$ under $D$ can be decomposed as:
   \begin{equation}
       \label{eq:freq-decomposition}
       f_{t_1:t_2}^{(\phi,\gamma)}(\cdot  \vert D)=
       \frac{\Gamma_{t_1:t}}{\Gamma_{t_1:t_2}}f_{t_1:t}^{(\phi,\gamma)} (\cdot  \vert
       D)+ \frac{\Gamma_{t:t_2}}{\Gamma_{t_1:t_2}}f_{t:t_2}^{(\phi,\gamma)}(\cdot  \vert D)
   ,\end{equation} 
   with the convention that $f_{t_1:t_1}^{(\phi,\gamma)}(\cdot  \vert D)$ is equal to
   zero. Let $h_t$ be a history of length less than $T$ and $\pi$ be an arbitrary
   policy. If $\Gamma_{t_1:T}$ is equal to zero then $\mathcal{A}_{\pi}^{\star}(h_t)$ is
   equal to $\mathcal{A}$ (hence $\mathcal{F}_{\pi}^{\star}(h_t)$) since the
   feature-action pairs visited after time $t$ do not contribute to
   $(\phi,\gamma)$-frequency. Otherwise, if $\Gamma_{t_1:T}$ is non-zero then for each action $a$,
\begin{subequations}
    \label{eq:arglub}
\begin{align}
    a&\in \mathcal{F}_{\pi}^{\star}(h_t)\\
    \iff a&\in \text{arglub}_{\preceq_{\circ}} \{f^{(\phi,\gamma)}_{t:T}(\cdot  \vert
    D^{\pi}(h_t,a')):\; a'\in \mathcal{A}\}\\
    \implies a&\in \text{arglub}_{\preceq_{\circ}} \{
        \frac{\Gamma_{0:t}}{\Gamma_{0:T}}f_{0:t}^{(\phi,\gamma)}(\cdot  \vert
        D^{\pi}(h_t,a'))+
        \frac{\Gamma_{t:T}}{\Gamma_{0:T}}f_{t:T}^{(\phi,\gamma)}(\cdot  \vert
        D^{\pi}(h_t,a')):\; a'\in \mathcal{A}\}\\
        \iff a&\in \text{arglub}_{\preceq_{\circ}} \left\{
        f^{(\phi,\gamma)}(\cdot  \vert D^{\pi}(h_t,a')):\; a'\in \mathcal{A}
    \right\}\\
            \iff a&\in \text{arglub}_{\preceq} \{D^{\pi}(h_t,a'):\; a'\in
            \mathcal{A}\}\\
                \iff a&\in \mathcal{A}_{\pi}^{\star}(h_t) 
.\end{align}
\end{subequations}
In going to the second line we have re-written the definition of
$\mathcal{F}_{\pi}^{\star}(h_t)$. In the third line we have used the fact that
$\preceq_{\circ}$ is consistent and 
that for any actions $a$ and $a'$, the distributions $f_{0:t}^{(\phi,\gamma)}(\cdot  \vert
D^{\pi}(h_t\cdot a))$ and $f_{0:t}^{(\phi,\gamma)}(\cdot  \vert
D^{\pi}(h_t\cdot a'))$ are equal, since the frequency of feature-action pairs in
between $0$ and $t$ depends only on $h_t$. In the fourth line, we have re-written each distribution in the set using
Equation~\ref{eq:freq-decomposition}. The fifth line follows from the fact that
$\preceq$ embeds into $\preceq_{\circ}$ via $(\phi,\gamma)$-frequencies and
the sixth line follows by definition of $\mathcal{A}_{\pi}^{\star}(h_t)$ in
Definition~\ref{def:Api}.
\end{proof}

\textbf{F2.} Let $(\mathcal{O},\mathcal{A},T,e,\preceq)$ be Direct Preference Process and
    $\preceq_{\circ}$ a total consistent preorder on
    $\text{Dist}(\mathcal{X}\times \mathcal{A})$ such that  $\preceq$ embeds into $\preceq_{\circ}$ 
    via $(\phi,\gamma)$-frequency.
    If the Markov Feature Assumption is satisfied then for any optimal policy $\pi$ and any two attainable $t$-histories $h_t,h_t'$ of length less than $T$, if $\phi(h_t)$ is equal to $\phi(h_t')$ then
            $\mathcal{F}_{\pi}^{\star}(h_t)$ is equal to
            $\mathcal{F}_{\pi}^{\star}(h_t')$. 

\begin{proof}We show by induction that the following stronger statement holds for
each non-negative integer $t$ less than $T$:

\textit{(P) For any optimal policy $\pi$ and attainable histories $h_t,h_t'$ that
    map to the same feature state, $\mathcal{F}_{\pi}^{\star}(h_t)$ is equal to
$\mathcal{F}_{\pi}^{\star}(h_t')$ and for any action $a$, $f_{t:T}^{(\phi,\gamma)}(\cdot  \vert
D^{\pi}(h_t\cdot a))$ is equal or $\preceq_{\circ}$-equivalent to $f_{t:T}^{(\phi,\gamma)}(\cdot
\vert D^{\pi}(h_t'\cdot a))$.}

Note that equality does
not imply $\preceq_{\circ}$-equivalence when $\Gamma_{t:T}$ is equal to zero, since $f_{t:T}^{(\phi,\gamma)}(\cdot  \vert
D^{\pi}(h_t\cdot a))$ is not defined as a distribution over feature-action pairs in
this case.

\textit{Base case: $t=T-1$}. Let $\pi$ be an optimal policy and $h_{T-1},h_{T-1}'$ be two
attainable histories such that $\phi(h_{T-1})$ is equal to $\phi(h_{T-1}')$. If $\gamma_{T-1}$ is equal
to zero then both $\mathcal{F}_{\pi}^{\star}(h_{T-1})$ and
$\mathcal{F}_{\pi}^{\star}(h_{T-1}')$ are equal to $\mathcal{A}$, and for each
action $a$, $f_{T-1:T}^{(\phi,\gamma)}(\cdot  \vert D^{\pi}(h_t\cdot a))$ is equal
to $f_{T-1:T}^{(\phi,\gamma)}(\cdot  \vert D^{\pi}(h_t'\cdot a))$. Otherwise, if
$\gamma_{T-1:T}$ is non-zero it follows
from the definition of $(\phi,\gamma)$-frequency
(Definition~\ref{def:discounted-freq}) that
both $\mathcal{F}_{\pi}^{\star}(h_{T-1})$ and $\mathcal{F}_{\pi}^{\star}(h_{T-1}')$ are
equal to the set \[
    \text{arglub}_{\preceq_{\circ}} \{\delta(\phi(h_{T-1}),a'):\; a'\in \mathcal{A}\}
,\] 
where $\delta(\phi(h_{T-1}),a')$ is the Dirac distribution over $\mathcal{X}\times \mathcal{A}$
concentrated at $(\phi(h_{T-1}),a')$.

\textit{Induction step.} For a fixed non-negative integer $t$ less than $T-1$,
assume that (P) holds at time  $t+1$. Let $\pi$ be an optimal policy and $h_t,h_t'$
be two attainable histories such that $\phi(h_t)$ is equal to $\phi(h_t')$. If $\Gamma_{t_1:T}$
is equal to zero then both $\mathcal{F}_{\pi}^{\star}(h_t)$ and
$\mathcal{F}_{\pi}^{\star}(h_t')$ are equal to $\mathcal{A}$. When $\Gamma_{t_1:T}$ is
non-zero, by the Markov Feature Assumption we have that $\phi(h_t\cdot (a,o))$ is equal to $\phi(h_t'\cdot (a,o))$
for each action-observation pair $(a,o)$. Thus, it follows from our inductive assumption
that if $h_t\cdot (a,o)$ is attainable and $\gamma_{t+1:T}$ is non-zero, then
\begin{equation}
    \label{eq:ia}
    f_{t+1:T}^{(\phi,\gamma)}(\cdot  \vert D^{\pi}(h_t\cdot
    (a,o)))\sim_{\circ} f_{t+1:T}^{(\phi,\gamma)}(\cdot  \vert
    D^{\pi}(h_t'\cdot (a,o)))
.\end{equation} 
Therefore, for each action $a$, 
\begin{align*}
    f_{t:T}^{(\phi,\gamma)}(\cdot  \vert D^{\pi}(h_t\cdot a))&=
    \frac{\gamma_t}{\Gamma_{t:T}}\delta(\phi(h_t),a)+
    \frac{\Gamma_{t+1:T}}{\Gamma_{t:T}}\sum_{o\in \mathcal{O}} \rho(o \vert h_t,a) f_{t+1:T}^{(\phi,\gamma)} (\cdot
    \vert D^{\pi}(h_t\cdot (a,o)))\\
&=
\frac{\gamma_{t}}{\Gamma_{t:T}}\delta(\phi(h_t'),a)+\frac{\Gamma_{t+1:T}}{\Gamma_{t:T}} \sum_{o\in \mathcal{O}}\rho(o \vert h_t',a) f_{t+1:T}^{(\phi,\gamma)}(\cdot \vert D^{\pi}(h_t\cdot (a,o)))\\
&\sim_{\circ}
\frac{\gamma_t}{\Gamma_{t:T}}\delta(\phi(h_t'),a)+\frac{\Gamma_{t+1:T}}{\Gamma_{t:T}}\sum_{o\in \mathcal{O}}\rho(o \vert h_t',a)f_{t:T}(\cdot  \vert D^{\pi}(h_t'\cdot (a,o)))\\
&= f_{t:T}^{(\phi,\gamma)}(\cdot  \vert D^{\pi}(h_t'\cdot a))
.\end{align*}
In the first line we have expanded the left hand side using
Equation~\ref{eq:freq-decomposition}. In the second line we've rewritten $\phi(h_t)$
as $\phi(h_t')$ and used the Markov Feature Assumption. In the third line we've used
Equation~\ref{eq:ia} and the fact that $\preceq_{\circ}$ is consistent. The fourth
line follows from Equation~\ref{eq:freq-decomposition}. These equations are
sufficient to show that (P) holds at time $t$, concluding our induction step.
\end{proof}

\end{document}